\documentclass[runningheads]{llncs}
\usepackage{subcaption}
\usepackage[T1]{fontenc}

\usepackage{graphicx}

\usepackage{minted}
\usepackage{algorithm}
\usepackage{algpseudocode}
\usepackage{listings}
\usepackage{ulem}
\usepackage{mathtools}
\usepackage{amssymb}
\usepackage{comment}
\usepackage{hyperref}
\usepackage{orcidlink}
\usepackage{booktabs}

\newcommand{\Rule}{\tau}
\newcommand{\softmatch}{\textit{softmatch}}

\renewcommand{\emph}[1]{\textit{#1}}

\renewcommand{\emph}[1]{\textit{#1}}

\renewcommand{\emph}[1]{\textit{#1}}
\newcommand{\nsqcap}{\mathrel{\ooalign{$\sqcap$\cr\hidewidth$/$\hidewidth}}}
\newif\ifshowannotations
\showannotationstrue   

\ifshowannotations
  \newenvironment{annotated}{}{}
\else
  \excludecomment{annotated}
\fi

\begin{document}
\title{Evaluation and Comparison Semantics for ODRL}

\author{Jaime Osvaldo Salas
\orcidlink{0000-0002-9353-8955}
\and
Paolo Pareti
\orcidlink{0000-0002-2502-0011}
\and
Semih Yumuşak
\orcidlink{0000-0002-8878-4991}
\and
Soulmaz Gheisari
\orcidlink{0000-0001-8974-2841}
\and
Luis-Daniel Ibáñez
\orcidlink{0000-0001-6993-0001}
\and
George Konstantinidis
\orcidlink{0000-0002-3962-9303}
}

\institute{University of Southampton, University Road, Southampton, SO17 1BJ, UK \\ Contact: \email{p.pareti@soton.ac.uk}}





\authorrunning{Salas et al.}

\maketitle             

\begin{abstract}
We consider the problem of evaluating, and comparing computational policies in the Open Digital Rights Language (ODRL), which has become the de facto standard for governing the access and usage of digital resources. Although preliminary progress has been made on the formal specification of the language's features, a comprehensive formal semantics of ODRL is still missing. In this paper, we provide a simple and intuitive formal semantics for ODRL that is based on query answering. Our semantics refines previous formalisations, and is aligned with the latest published specification of the language (2.2). Building on our evaluation semantics, and motivated by data sharing scenarios, we also define and study the problem of comparing two policies, detecting equivalent, more restrictive or more permissive policies.


\end{abstract}
\section{Introduction}

The Open Digital Rights Language (ODRL) is a W3C Recommendation that aims at providing a flexible and interoperable way to express rights and obligations related to digital content. Notably, ODRL has been adopted by developers of Data Spaces\footnote{https://digital-strategy.ec.europa.eu/en/policies/strategy-data} to regulate access to datasets, and is set to play a key role in regulating data usage within AI pipelines. 
However, to be useful in practice, policy specifications need to be implementable, enforceable, and universally understandable within the Data Spaces ecosystem. The first step in implementing any policy is deciding the semantics of the policy evaluation problem: given a usage policy over a resource, e.g., a dataset, and a \emph{state of the world} capturing the operations of actors, evaluate whether the policy is violated or satisfied in this state.

At the same time, a major function of Data Spaces is data exchange in scenarios where a data consumer wants to use data accompanied by a usage policy (written by the data provider). In this paper, we model the request of the consumer as another ODRL policy and define the problem of policy comparison: given a provider policy defining usage constraints over a resource, a usage request policy describing intended actions over that resource, decide if the request conflicts with provider policy, before the exchange takes place. 

These scenarios require a precise formal semantics for ODRL evaluation and comparison. ODRL's origins come from the area of Deontic Logic~\cite{follesdal1971deontic}, which provides the formal tools to define such a language. However, as of August 2025, the formal semantics of ODRL is still in a draft state\footnote{https://w3c.github.io/odrl/formal-semantics/}, and most importantly, it is in natural language. Only a few academic works have attempted to define semantics for ODRL \cite{pucella2006formal,steyskal2015towards,de2019odrl,bonatti2025towards}, that either consider older versions of ODRL, or are at a preliminary and partial state of formalisation. A more detailed comparison with previous work is provided in Section \ref{sec:relateds-work}. 

The problem of evaluation is arguably the most fundamental problem of computational policies, but the fact that it has not been treated in a uniform way is not surprising: ODRL has been used in a large variety of very different applications and each has its own description of operations, systems and functions that the policies should be checked against. Each such system has implemented its own ad-hoc policy evaluation , having to closely synchronise with other systems to achieve interoperability. In this paper, and in line with the recently published draft semantics, we  model a state of the world as a set of \emph{events} and implement it as a database relation. Although implementation is orthogonal (e.g., one could use an event recognition system within an organisation, or simply a triplestore with a log of events) our mathematical representation of this is a simple relation that we can query to decide our evaluation and comparison problems.  

Our main contributions are the following: (1) a novel, intuitive and simple formal semantics for ODRL, and for a simpler language that we call ODRL Lite,  based on translating policies to queries on a state of the world and (2) a semantics for policy comparison grounded on the notion of query containment and equivalence - giving rise to different comparison semantics (symmetric and asymmetric) based on the application. 

Another aspect of ODRL semantics that is not strictly defined is how to handle reasoning, which leads to many possible interpretations being possible. 
In this paper, instead of committing our semantics to a specific approach,  we chose to decouple the reasoning dimension and show how our semantics can easily be extended with different types of inferences through a separate materialisation step. This allows us to provide a core semantics for the basic case without inference, and the flexibility to adapt to different reasoning algorithms.

\section{Background}

ODRL is a policy expression language that provides a flexible and interoperable information model, vocabulary, and encoding mechanism to represent \emph{rules} about the usage of content and services. ODRL policies have four \emph{core components}: (1) \emph{Asset}: a resource that is subject of a policy; (2) \emph{Party}: an entity that undertakes functional roles as the \emph{assigner} or \emph{assignee} of a policy; (3) \emph{Action}: an operation (for example Read, Write, Copy, Anonymise) that can be exercised on a target asset and (4) \emph{Rule}: defines the \emph{permission}, \emph{prohibition}, or \emph{obligation} of the assignees of the policy to execute actions over assets. A single policy may define multiple rules on different core components.

Components of an ODRL Policy can be constrained. Constraints are boolean expressions that either dictate conditions applicable to a Rule, or refine the semantics of actions, parties or assets. In ODRL, the former usage is called \emph{Constraints}
 and the latter  \emph{Refinements}. For example, a permission rule can be constrained to a particular datetime, or to a particular location (e.g., the European Union); the Action \emph{Print} can be refined with the boolean expression \emph{with resolution less than 500dpi}; and the asset \emph{Document D} can be refined with \emph{subset of pages 1 to 10}.  

Complex ODRL policies can also make use of additional types of rules, which build on the three basic types of rules previously mentioned. Permissions can have \emph{duties}, which are types of obligations that need to be fulfilled before the assignee exercises the permission. The \emph{remedy} and \emph{consequence} rules specify remedial actions that can restore validity in a state of the world in the case of prohibition violations, or in the failure to fulfill an obligation or a duty. 

ODRL policies are usually read as \emph{\underline{Rule} from \underline{Assigner} for \underline{Assignee} to \underline{Action} \underline{Asset} subject to \underline{Constraints} and \underline{Refinements}}, \emph{e.g.} \emph{Permission from Bob for Alice to Read Document D subject to the constraint the day being Monday and Document D to be refined to its Subset of Pages 1-10}. While all rules must define the Action component, all other elements are optional. 
We assume all instances of components are drawn from a known vocabulary or Knowledge Graph.

Policies act upon a \emph{world}, e.g., a Socio-Technical System or a Data Space. A \emph{state of the world} is a formal representation of the state of the assets, parties and any other environmental information relevant for the constraints. For example, \emph{Bob Read document D, Alice Read Document D, The day is Tuesday.} An ODRL Evaluator receives as input an ODRL policy and a state of the world and decides which rules of the policy are violated or fulfilled in the state of the world. An ODRL Evaluator may work in two different scenarios or modalities: \emph{Access Control}, adds to the input an action that is \emph{requested}, the evaluator must decide if the execution of the action would create a state of the world that violates any policy; and \emph{Monitoring}, where the state of the world includes actions already performed and the evaluator decides violation/fulfillment. In this paper, we focus on the problem of violation and fulfillment of policies given a state of the world, conforming to the Monitoring scenario. Note however, that the Access Control modality can be reduced to the Monitoring one if the state of the world that would result after performing the requested action can be predicted.

\section{First-order ODRL Semantics}

An ODRL policy is evaluated against a \textit{state of the world}, which we model as a set of \emph{events}. For a state of the world $\omega$ and an event $e$ we write $\omega(e)$ when $e \in \omega$. An event is described by a finite number of features $i \in [0,n]$, corresponding to core components, constraints and refinements of the language, e.g., time, actor, action, etc. These features can also describe relationships with other events. For example, a feature can represent the number of times a similar action has been already performed, as defined by the ODRL \emph{count} constraint.
Thus an event $e$ can be defined as an n-tuple $<e_0, e_1, e_2, \ldots, e_n>$, where the value of a feature $e_i$ is either null (unspecified), a constant, or a set of constants. 
For simplicity, we assume that the first value $e_0$ of event $e$ is always the timestamp of when the action has been performed, and $e_1$ is the action. 
All other values correspond to the rest of the ODRL components, refinements and constraints an ODRL rule can have.
%
%
Each feature is unique, and its relationship with a specific ODRL component known. Two features $i$ and $j$ might both represent the  ``is located in'' relation, but the first one might represent this relation as a property (refinement) of the assignee component, while the second one as a refinement of the asset.  For any feature $i$, let $\gamma_{i}$ be the ODRL component of $i$, that is, if $i$ represents a refinement of the value $e_k$, then $\gamma_{i} = k$.  For example, if $i$ is ``has the creation date'' of the asset component, then $\gamma_i=Asset$.
When a feature captures a rule-wide property (corresponding to a constraint and not a refinement), or the timestamp $e_0$, then $\gamma_{i}$ equals a special constant $\rho$. In all other cases, that is, when a feature $i$ is a component itself, then simply $\gamma_{i} = i$.

{\setlength{\tabcolsep}{6pt}
\begin{table}[t]
   \centering
\begin{tabular}{llllll}
\toprule
Datetime & Action & Actor & Asset & Print.Resolution & Book.Pages \\
\midrule
1 & Print & Alice & Picture & 500 dpi & null \\
2 & Read & Bob & Book & null & 450 \\
3 & Print & Alice & Book & 600 dpi & 300 \\
\bottomrule
\end{tabular}
\caption{A state of the world corresponding to Example \ref{example1}}
\label{StateOfTheWorld}
\end{table}
}

\begin{example}
\label{example1}
    Consider a world with actors Alice and Bob, assets Book and Picture, and actions Read and Print. Book has the property ``number of pages'', and action Print has the property ``resolution'' that can take an integer value in the range 100-5000dpi. 
    Table \ref{StateOfTheWorld} shows the model of the events ``Actor Alice executes Action Print at 500dpi on Asset Picture at datetime 1'', ``Actor Bob executes Action Read on Asset Book with 450 pages at datetime 2'' and ``Actor Alice executes Action Print at 600 dpi on Asset Book with 300 pages at datetime 3''.
\end{example}

Rules in an ODRL policy apply only to specific events, based on restrictions on the values of the features of an event. For example, a prohibition to ``print'' can be restricted to apply only to print actions over certain resolutions. To restrict the value of a feature $i$ we write triples of the form  $<i, op, v>$, where $i$ is one of the $n$ features of an event, $v$ is a constant or set of constants, and $op$ is one of the following binary comparison operators:  \texttt{odrl:equals} ($=$), \texttt{odrl:gt} ($>$), \texttt{odrl:gteq} ($\geq$), \texttt{odrl:lt} ($<$), \texttt{odrl:lteq} ($\leq$), \texttt{odrl:neq} ($\neq$), \texttt{odrl:isA} (or \texttt{rdf:type}), \texttt{odrl:hasPart} ($\supset$), \texttt{odrl:isPartOf} ($\subset$), \texttt{odrl:isAllOf} ($\equiv$), \texttt{odrl:isAnyOf} ($\in$), and \texttt{odrl:isNoneOf} ($\notin$). A \emph{condition} $C$ is either a single triple $<i, op, v>$, known as a \emph{simple} condition, or a \emph{complex} condition which is a boolean combination of simple conditions, combined with the $\wedge$, $\vee$ and $\neg$ operators.\footnote{We omit the \texttt{odrl:andSequence} logic operator due to its lack of a clear semantic interpretation.} 
We denote with $C(e)$ the boolean result of the evaluation of condition $C$ on event $e$. Note that, ODRL does not provide negation of conditions directly but provides the exclusive-or operator $\nleftrightarrow$ that we can use to obtain negation when features are not null.\footnote{For a condition $c$, holds $\neg c \equiv c \nleftrightarrow \top$ where $\top$ is always true; $\top$ can be expressed with a disjunction of complementary conditions, such as $<i,=, x> \vee <i,\neq, x>$.}
A complex condition is evaluated in the natural way as a boolean combination of simple conditions. A simple condition $<i, op, v>$ is evaluated to true on an event $e$ if the value comparison of $e_i$ and $v$ using operator $op$ is true, and false otherwise. We ground our condition evaluation on SPARQL Filter Evaluation semantics, by defining the value comparison as true if the equivalent SPARQL filter expression would keep the bound solution. Thus evaluation errors, for example due to incomparable types, behave as evaluations to false. Condition $<i, \texttt{odrl:isA},v>$ is evaluated as $<i_c, \texttt{odrl:hasPart},\{v\}>$, where $i_c$ is the set of classes of $i$. The remaining set operators are evaluated in their natural way.
When the value of feature $i$ is unspecified (i.e., null) in the event, then any simple condition $C$ for that feature evaluates to false. 

Corresponding to ODRL rules, an \emph{event rule} $\tau$ is defined as a conjunction of conditions. In the next subsection we will see how event rules are assigned a specific type, such as permissions or prohibitions, when used within a policy.

\begin{annotated}
\begin{example}
\label{example2}
    Consider a policy with the following rules (1) ``Permission for Alice to Print Picture'', is translated as permission 
    $<$$Actor,=,Alice$$>$ $\land$
    $<$$Action,=,Print$$>$ $\land$
    $<$$Asset,=,Picture$$>$; (2) ``Prohibition for Bob to Read Book between datetimes 3 and 5 and number of pages is greater than 250'' translates as prohibition
    $<$$Actor,=,Bob$$>$ $\land$
    $<$$Action,=,Read$$>$ $\land$
    $<$$Asset,=,Book$$>$ $\land$
    ($<$$Datetime,\leq,5$$>$ $\land$ $<$$Datetime,\geq,3$$>$) $\land$
    $<$$Book.Pages,>,250$$>$ (demonstrating a complex condition on the Datetime feature); and (3) ``Obligation for Bob to Read Book before datetime 3'' is translated as obligation
    $<$$Actor,=,Bob$$>$ $\land$
    $<$$Action,=,Read$$>$ $\land$
    $<$$Asset,=,Book$$>$ $\land$
    $<$$Datetime,<,3$$>$.
\end{example}
\end{annotated}

Given a condition $C$, let $\bar{C}$ be the set of simple conditions in $C$, and $I_C$ be the set of features that appear in it, that is, the set $\{i \; | \; \exists \; op, \; v \; . \;  <i, op, v> \in \bar{C}  \}$. 
Correspondingly, given an event rule $I_\tau$, let $I_{\tau} = \bigcup_{C \in \tau} I_C$.
To align with the ODRL specification we require all our rules to be well-formed, as follows. 
An event rule $\tau$ is said to be \emph{well-formed} if:
\begin{enumerate}
    \item there is an action $a$ such that $<Action, =, a> \in \tau$; that is, each ODRL rule must specify the action $a$ it applies to,
    \item for all $i \in I_\tau$, if $\gamma_i \in \{Action, Asset, Party \}$, then $<\gamma_i, =, v> \in \tau$, and there exists no other condition $C \in \tau$ such that $\gamma_i \in I_C$
    (in other words, each ODRL core component used in a rule must be specified once, and not used in any other condition),
    \item for all complex conditions $C \in \tau$ there do not exist $i$ and $i^{\prime}$ in $I_C$ with $\gamma_i \neq \gamma_{i^{\prime}}$; in other words, each boolean combination of constraints must use features of a single ODRL component. 
\end{enumerate}
From now on, we will assume that event rules are well-formed.
We say that an event rule $\tau$ \emph{matches} an event $e$, denoted by $\mathsf{match}(\tau, e)$ if all its conditions evaluate to true on $e$, that is when $\bigwedge_{C \in \tau}C(e)$ is true.
Note that $\mathsf{match}(\tau, e)$ has a natural correspondence to the evaluation of a boolean query on tuple $e$, and we can easily implement it with SQL or SPARQL.

\begin{annotated}
\begin{example}
\label{example3}
All three rules from  Example \ref{example2} are well-formed. The rule permitting Alice to Print Picture matches the event with datetime 1 from Table \ref{StateOfTheWorld}. The same rule does not match the event with datetime 2, because, among other things, Action and Actor don't match, nor with the event with datetime 3 because the Assets don't match. 
The rule prohibiting Bob to Read the Book does not match the event with datetime 1 or 3 because either the Action and Actor don't match. The rule does not match with the event with datetime 2 either: though the event specifies $Book.Pages$ must be greater than 250, its datetime falls outside the specified range.
Finally, the rule obligating Bob to Read the Book does not match the first and third event, but it does match the second.
\end{example}
\end{annotated}

\subsection{ODRL Lite}
 
An \emph{ODRL Lite policy} is a tuple $<P,F,O>$, where $P$, $F$ and $O$ are the sets of permission, prohibition and obligation event rules, respectively.
Faithful to the deontic semantics of these concepts permissions model what is allowed in a state of the world, but not compulsory, prohibitions state what is not allowed, and obligations require certain actions to be performed to be fulfilled.
\begin{definition}
\label{def:violationlite}
We say that a state of the world $\omega$ violates a policy $<P,F,O>$ if any of the following three cases occurs: 
\begin{itemize}
    \item (\textbf{permissions}) there exists an event $e \in \omega$ and for all permission rules $\Rule \in P$,  $match(\Rule,e)$ is false; equivalently if the following is true: 
    \begin{equation}
    \label{eq:permissions}
        \exists e\ \omega(e) \wedge \bigwedge_{\tau \in P} \neg match( \tau,e)
    \end{equation}    
    \item (\textbf{prohibitions}) there exists an event $e \in \omega$ and a prohibition rule $\tau \in F$, such that  $match(\Rule,e)$ is true; equivalently if the following is true:
    \begin{equation}
    \label{eq:prohibitions}
        \exists e\ \omega(e) \wedge \bigvee_{\tau \in F} match( \tau, e)
    \end{equation}

    \item (\textbf{obligations}) for any of the obligation rules $\Rule \in O$, there does not exist an event $e \in \omega$, such that $match(\Rule,e)$ is true; equivalently:
    \begin{equation}
    \label{eq:obligations}
    \bigvee_{\tau \in O}  \nexists e\ \omega(e) \wedge match(\tau,e)
    \end{equation}

\end{itemize}
\end{definition}

Conversely, a state of the world $\omega$ is said to be \emph{valid} with respect to a policy $p=<P,F,O>$ if it does not violate it, or equivalently, the policy is \emph{satisfied} or \emph{true on $\omega$}, writing $p(\omega) = T$.
Following from Definition~\ref{def:violationlite}, to decide violations one should evaluate the following logical expression:

\begin{equation}
\label{eq:violationlite}
\begin{aligned}
    \left( \exists e\ \omega(e)\wedge 
        ( \bigwedge_{\tau \in P}\neg match(\tau,e))  
        \vee
        ( \bigvee_{\tau \in F}match(\tau, e)) \right) \\ 
        \vee \left( \bigvee_{\tau \in O}  \nexists e\ \omega(e)\wedge match(\tau,e) \right)
\end{aligned}
\end{equation}

Since the predicate $match$ is just a shorthand notation for the satisfaction of the complex conditions of a rule on an event tuple, as discussed above, expression~\ref{eq:violationlite} is just a first-order query on relation $\omega$ (the state of the world) and can be easily implemented via SQL or SPARQL. 

\begin{annotated}
\begin{example}
\label{example4}
Consider the policy $p = <{P},{F},{O}>$ where each of $P$, $F$, $O$ contain the single event rule $p_1$, $f_1$ and $o_1$, respectively, as defined in Example~\ref{example2}. 
Permission $p_1$ allows Alice to Print Picture, $f_1$ prohibits Bob to Read Book at a certain time and over a number of pages, and $o_1$ obliges Bob to Read Book before a timestamp. 
The evaluation of $p$ over the state of the world $\omega$ in Table \ref{StateOfTheWorld} is as follows: for $p_1$, the matches with events at datetime 2 and 3 return false, meaning $\omega$ violates the permissions of the policy (e.g.\ the policy does not explicitly grant permission for Bob to Read Book); for $f_1$, none of the events match as explained in Example~\ref{example3}, hence, $\omega$ is valid wrt the prohibitions of the policy; for $O_1$, there is an event where Bob Read[s] Book before datetime 3, so $\omega$ is valid wrt the obligations of the policy.
However, as permissions are violated, $\omega$ violates the whole policy.
\end{example}
\end{annotated}

One can observe that our semantics adopts the \emph{prohibited-by-default} approach, that is, everything not explicitly permitted is implicitly prohibited. Prohibitions are then primarily useful to ``carve out'' prohibited subsets of a permission. We believe this is a practical choice in several use cases, such as data sharing, which require a cautious approach to permissions. However,  one can easily simulate the alternative semantics, \emph{permitted-by-default}, by adding to all policies a universal permission that matches every event. This can be easily done with inferences over actions, by defining a single permission for a generic top-level action that all other actions are specialisations of. The only limitation of simulating this alternative semantics is that it does not allow the ``carving out'' of permitted subsets of prohibitions.

It should also be noted that our semantics is \emph{closed-world}; any event that does not appear in the state of the world is false, that is, is inferred not to have happened. Closed-world semantics are more natural for a usage policy, which by nature intends to govern the access and use of resources.

\subsection{Full ODRL}

In this section we extend ODRL Lite to cover useful ODRL features such as duties, remedies and consequences to express the core of the ODRL language.

\begin{definition}
\label{def:odrlfull}
An ODRL Lite policy $<P,F,O>$ can be extended into an \textit{ODRL Policy} $<P,F,O, DP, DPC, FR, OC>$, where:
\begin{itemize}
    \item $DP$ is a set of pairs of event rules $<$$\tau,\tau'$$>$, with $\tau, \tau' \in P$, specifying that $\tau'$ is a duty of permission $\tau$.
    \item $DPC$ is a set of triples of event rules $<$$\tau,\tau',\tau''$$>$, with $\tau, \tau', \tau'' \in P$ specifying that $\tau'$ is a duty of permission $\tau$ with consequence $\tau''$.
    \item $FR$ is a set of pairs of event rules $<$$\tau,\tau'$$>$, with $\tau \notin F$ and $\tau' \in P$ specifying that $\tau'$ is a remedy of prohibition $\tau$.
    \item $OC$ is a set of pairs of event rules $<$$\tau,\tau'$$>$, with $<Datetime, \leq, t> \in \tau$, where $t$ is a timestamp, $\tau \notin O$, and $\tau' \in P$, specifying that $\tau'$ is a consequence of not fulfilling the obligation $\tau$ by time $t$.
\end{itemize}
\end{definition}

Note that per the specification "the Party obligated to perform the duty MUST have the ability to exercise the Duty Action"\cite{ODRL22}. Thus, every duty rule $\Rule'$ in a pair $<$$(\Rule, \Rule')$$>$ in $DP$ must explicitly be in $P$ (and this sentence is guaranteed for consistent policies). We have made similar choices for the rest of these complex rules as per the specification.

Unlike ODRL Lite, this extended definition makes direct use of event timestamps, to establish the temporal ordering of events.
Intuitively, $DP$ describes duties that must have been implemented before a permission was exercised; $DPC$ extends this with consequences that must come, alongside the original duty if a permission was exercised without fulfilling its duty; $FR$ contains prohibitions and remedies which must follow the violation of the former; lastly, $OC$ dictates consequences that should occur, in addition to the obliged event, if the latter is not performed in time. 
%
Given a rule $\Rule$ and an event $e$, a time-independent match is denoted $\softmatch(\Rule,e)$ and is true if $match(\Rule \setminus \{\, \langle Datetime, \leq, t \rangle \mid t \in T \,\},e)$, where $T$ is the set of all time points.
\begin{definition}
\label{def:violationodrl}
   An ODRL Policy $<P,F,O, DP, DPC, FR, OC>$ is violated on a state of the world $\omega$ if the ODRL Lite Policy $<P,F,O>$ is violated on $\omega$, or one of the following holds:
   \begin{itemize}
       \item (\textbf{permission duties}) there exists an event $e \in \omega$ and a pair of event rules $<$$\Rule,\Rule^{\prime}$$>$ in $DP$ such that  $match(\Rule,e)$ is true and there is no event $e^{\prime} \in \omega$ with $e'_0 \leq e_0$ and $match(\Rule',e')$ true; equivalently:
       \begin{equation}
       	\label{eq:permduty}
        \exists e\ \omega(e) \wedge \bigvee_{<\Rule,\Rule'> \in DP} match(\tau,e) \wedge \nexists e'\ \omega(e') \wedge match(\tau',e') \wedge (e'_0 \leq e_0)
         \end{equation} 
         \item (\textbf{permission duties with consequences}) 
         there exists an event $e \in \omega$ and a triple of rules $<$$\Rule,\Rule^{\prime},\Rule^{\prime\prime}$$>$ in $DPC$ such that  $match(\Rule,e)$ is true  (permission exercised), and (a) there is no prior (duty) event $e' \in \omega$ with $match(\Rule',e')$ and $e'_0 \leq e_0$ and (b) there is no subsequent duty event  $e' \in \omega$ with $match(\Rule',e')$ and $e_0 \leq e'_0$, or, there is no subsequent consequence event $e'' \in \omega$ with $match(\Rule'',e'')$ and $e_0 \leq e''_0$;
         equivalently:
\begin{equation}
\small
\label{eq:permdutycon}
\exists e\ \omega(e) \wedge \underset{<\Rule,\Rule', \Rule''> \in DPC}{\bigvee} match(\Rule,e)  \wedge \left(
    \begin{array}{@{}c@{}}
         \nexists e'\ \omega(e') \wedge  match(\Rule',e') \wedge (e'_0 \leq e_0)\ \wedge   \\
        \left(
            \begin{array}{c}
            \nexists e'\ \omega(e') \wedge match(\Rule',e') \wedge (e_0 \leq e'_0) \\
            \vee \\
            \nexists e''\ \omega(e'') \wedge match(\Rule'',e'') \wedge (e_0 \leq e''_0)
            \end{array}
        \right)
    \end{array}
\right)
\end{equation}

        \item (\textbf{prohibition remedies}) there exists an event $e \in \omega$ and a pair of event rules $<$$\Rule,\Rule^{\prime}$$>$ in $FR$ with $match(\Rule,e)$ and there is no event $e' \in \omega$ with  $match(\Rule',e')$ and $e'_0 \geq e_0$; equivalently: 
        \begin{equation}
        \label{eq:prohibremedy}
        \exists e\ \omega(e) \wedge \bigvee_{<\Rule,\Rule'> \in FR} match(\tau,e) \wedge \nexists e'\ \omega(e') \wedge match(\tau',e') \wedge (e'_0 \geq e_0)
         \end{equation}
         \item (\textbf{obligation consequences}) for any pair of event rules $<$$\Rule,\Rule^{\prime}$$>$ in $OC$,
         such that $<$$Datetime, \leq, t$$> \in \tau$ there is no event $e \in \omega$ such that $match(\Rule,e)$ (obligation not fulfilled)  and there is no pair of events $e', e'' \in \omega$ such that $\softmatch(\Rule,e')$ (obligation implemented late), $match(\Rule',e'')$ and $e''_0 \geq t$ (consequence implemented); equivalently:
        \begin{align}
        \label{eq:obcon}
        \bigvee_{<\Rule,\Rule'> \in OC, <0, \leq, t> \in \Rule}\left(
            \begin{tabular}[t]{@{}c@{}}
         	 $  \big(\nexists e. \ \omega(e) \wedge match(\Rule, e)\big) \wedge \big(\nexists e', e'' .\  \omega(e') \wedge \omega(e'')$ \\         	 
         	 $  \wedge\ \softmatch(\Rule, e') \wedge match(\Rule', e'') \wedge (e''_0 \geq t) \big)$ \\
        	\end{tabular}
            \right)
        \end{align}
   \end{itemize}
\end{definition}

Again, all elements of Definition~\ref{def:violationodrl} are first-order queries and thus detecting violations of ODRL policies can be done via query answering.

\begin{annotated}
\begin{example}
\label{example5}
    Consider the following Policy : ``Permission for Alice to Print Book provided Bob has Read Book''. We translate this as a pair of event rules $<\tau,\tau'>$ where the permission component is $<Actor,=,Alice> \land
    <Action,=,Print> 
  \land <Asset,=,Book> $ and the duty component is translated as 
  $<Actor,=,Bob> \land
    <Action,=,Read> 
  \land <Asset,=,Book> $.
  In this case, each individual component matches an event in Table \ref{example1}-- the event with datetime 1 for the permission component, and the event with datetime 2 for the duty component. However the datetime feature of the duty component is greater than that of the permission component, indicating that the permitted event occurred before the duty was carried out. Thus, the state of the world violates the policy.
\end{example}
\end{annotated}

Note that equation~\ref{eq:permduty} does not examine if there is a permission for all events, as this is already performed in equation \ref{eq:permissions}. 
This is why any event detected here (and equally for should also be explicitly permitted by equation \ref{eq:permissions}; this is required by Definition~\ref{def:odrlfull} which makes sure that for all pairs $<$$(\Rule, \Rule')$$>$ in $DP$, rule $\Rule$ also is in $P$. The same is true for permission rules in $DPC$. This is not true however for prohibitions in $FR$, i.e., if a for a pair  $<$$(\Rule, \Rule')$$>$ a rule $\Rule$ is also in $F$ then equation~\ref{eq:prohibitions} will capture a violation independent of remedies, a case we want to avoid as it would make the remedy meaningless.
Similarly, an obligation rule $\Rule$ that exists in a pair $<$$(\Rule, \Rule')$$>$ in $OC$ should not exist in $O$ since the absence of an obliged event in $O$, inside the ``deadline'', would trigger a violation by equation~\ref{eq:obligations}, regardless of whether the obligation and its consequence were later fulfilled.

\subsection{Incorporating Reasoning}
Instances of ODRL components may come from an ontology that defines relationships between them that are relevant to the problem of policy evaluation. For example,  W3C's ODRL Vocabulary\footnote{\url{https://www.w3.org/TR/odrl-vocab}}  defines the Actions \textit{Use} and \textit{Transfer} as the top level types, with more specific instances that encompass their operational semantics specified with the \textit{included in} relation. For example the operational semantics of action \textit{Display} (e.g. an image) is included in the operational semantics of action \textit{Play} (e.g. a video). A permission to perform an action, such as \textit{Play}, naturally also implies the permission to perform its sub-actions, such as \textit{Display}. Our semantics can easily be extended to capture this type of reasoning, by considering extended policies where all such implied permissions are materialized. Other types of reasoning can equally be dealt with similar materialisation steps, or through extensions of our definitions of query evaluation, containment and equivalence to include reasoning dependencies.  For example, conjunctive query containment under dependencies~\cite{10.1145/588111.588138} can be answered either by (a) materialising new event rules, by the use of forward-chaining, algorithms that use the dependencies, resulting in saturated policies which can then be checked for containment; or (b) directly rewriting the query form of the initial policy into a new policy and then checking for containment~\cite{calvanese2007tractable}. 

We do not study other rules of inference in this paper, as no obvious interpretation of them can be extracted from the ODRL specification. For example, it is debatable whether a prohibition to perform the \textit{Play} action should also imply a prohibition to perform the \textit{Display} action. Conversely, a prohibition to \textit{Display} might also imply a prohibition to \textit{Play}. How reasoning should be dealt with is also unclear for obligations or other ODRL components. 
    
\section{Policy Comparison Semantics}
\label{sec:policycomparison}

In this section we model policy comparison using query containment and equivalence. In particular, query containment models policy implication, intuitively stating when a policy is broader or stricter than another policy. On the other hand, query equivalence models policies that are true on exactly the same states of the world, that is, they are not conflicting. A query $q$ is \emph{contained} in a query $q'$ ($q \sqsubseteq q'$) if for all database instances $D$, the result of $q$ on $D$, denoted $q(D)$ is included in the result of $q'$ on the same $D$, that is $q(D) \subseteq q'(D)$. When $q \sqsubseteq q'$ and $q' \sqsubseteq q$ we say that the two queries are equivalent and write $q \equiv q'$.

Given ODRL policies $p$ and $p'$, $p$ is contained in $p'$, denoted $p \sqsubseteq p'$, if: for any state of the world $\omega$, whenever $p(\omega)$ is true then $p'(\omega)$ is true.
Thus by translating the policies $p$ and $p'$ to queries, $q_{p}$, $q_{p'}$ correspondingly, as dictated by Def~\ref{def:violationodrl}, it holds that $p \sqsubseteq p'$ iff $q_{p} \sqsubseteq q_{p'}$. If $p \sqsubseteq p'$ and $p' \sqsubseteq p$ then we say that  $p$ is equivalent to $p'$, denoted $p \equiv p'$.

In the general case, any semantic difference between two policies could be considered a cause for conflict, such as in a policy negotiation scenario when each party is not willing to compromise on their chosen set of policy rules. 
We call this a \emph{symmetric conflict}. When a symmetric conflict exists between policies $p$ and $p'$, then it also exists between $p'$ and $p$.

\begin{definition}
    Given policies $p$ and $p'$, we say that there is a \emph{symmetric conflict} between $p$ and $p'$ if $p \not\equiv p'$.
\end{definition}

Often, it is useful to determine whether a policy is strictly more/less permissive than another one. For example, in a data sharing scenario, no conflict arises if the policy proposed by a data requester is stricter than the policy offered by the data provider.
This relaxation leads to the notion of asymmetric conflicts. 

\begin{definition}
    Given a requester policy $p^r$ and a provider policy $p^o$, we say that there is an \emph{asymmetric conflict} between them, if $p^r \not\sqsubseteq p^o$.
\end{definition}

\subsection{Comparison of ODRL Lite policies}\label{subsec:conflict}

We now provide an approach, alternative to query containment, to detect asymmetric conflicts (and thus also symmetric ones) among ODRL Lite policies that is simpler to implement.  
As evidence of this, we developed an implementation\footnote{\url{https://github.com/paolo7/ODRL2SHACL}} which is currently used in the UPCAST project \cite{yumusak2024data} to facilitate policy negotiation in Data Marketplaces. 
This approach works by directly comparing the rules of two policies against each other. For this, we are going to use the notions of \emph{event rule containment} and \emph{event rule overlap}.

\begin{definition}
\label{def:rulecont}
     For all event rules $\Rule$, $\Rule'$, we say that $\Rule$ is contained in $\Rule'$, denoted $\Rule \sqsubseteq \Rule'$ if for all events $e$, $match(\Rule, e)$ implies $match(\Rule', e)$.    
\end{definition}

\begin{definition}
\label{def:ruleoverlap}
     For all event rules $\Rule$, $\Rule'$, we say that $\Rule$ overlaps with $\Rule'$, denoted $\Rule \sqcap \Rule'$ if there exists an event $e$ such that $match(\Rule, e) \wedge match(\Rule', e)$.
\end{definition}

These definitions are trivially lifted to sets of event rules on either side of the containment and overlap operator. This approach also assumes that the policies are in a \emph{consistent} normal form, that is, when there is no overlap between prohibitions and permissions/obligations, and every event that matches an obligation also matches a permission. While the existing implementation assumes that policies are already in this format, any ODRL Lite policy can be made consistent by removing from a policy the following redundant elements: (1) every part of a permission and obligation that is explicitly forbidden, and (2) every part of an obligation that is not permitted. Note how these elements are redundant, and thus their removal does not change the meaning of the policy. Also note how prohibitions can be removed too, as they are redundant in a consistent policy.

\begin{definition}
\label{def:ruleconsist}
     An ODRL Lite policy $p = <$$P,F,O$$>$ is \emph{consistent} if $P \nsqcap F$, $O \nsqcap F$, and $O \sqsubseteq P$.   
\end{definition}

\begin{theorem}
\label{theorem:conflict}
    Given a consistent ODRL Lite requester policy $p = <$$P,F,O$$>$, and a consistent ODRL Lite provider policy $p' = <$$P',F',O'$$>$, $p \not\sqsubseteq p'$ iff:
\begin{enumerate}
    \item The requester's permissions are not contained in the provider's permissions, that is,  $P \not\sqsubseteq P'$ is true, or
    \item The requester is not explicitly agreeing to every obligation set by the provider, that is, $\exists \tau' \in O', \nexists \tau \in O \;   (\tau \sqsubseteq \tau')$

\end{enumerate}  
\end{theorem}

\begin{proof}

$\Leftarrow$ If $p \not\sqsubseteq p'$, then there must exist a state of the world $\omega$ where $p$ is valid but $p'$ is not. Since prohibitions are redundant in consistent policies, there must exist a violation for $p'$ in $\omega$ from the permission or obligation equation of Definition \ref{def:violationlite}. If the violation is caused by the \textbf{permission equation}, then there exists an event $e \in \omega$ that matches a permission in $P$ but no permissions in $P'$. Thus $P \not\sqsubseteq P'$. If the violation is caused by the \textbf{obligation equation}, then there exist an obligation $\tau' \in O'$ such that no event $e$ exists in $\omega$, such that $match(\tau', e)$ is true. Thus, there cannot exist a $\tau \in O$ such that $\tau \sqsubseteq \tau'$, as this would imply the existence of an $e$ in $\omega$ that matches $\tau$ and, by inclusion, $\tau'$.

$\Rightarrow$ If either condition (1) or condition (2) of the theorem are true, then we can construct a state of the world $\omega$ where $p$ is valid but $p'$ is not (thus proving $p \not\sqsubseteq p'$) as follows. If \textbf{the second condition of the theorem is true}, pick an obligation $\tau' \in O'$ such that $\nexists \tau \in O \   (\tau \sqsubseteq \tau')$. Construct $\omega$ from an empty set by adding an event $e$ to $\omega$ for each $\tau \in O$, such that $match(\tau,e)$ and $\neg match(\tau',e)$ (note that, by the premises, such element must exist). It is easy to see that $\omega$ satisfies $p$ but not $p'$, as it violates obligation $\tau'$. 
If \textbf{the second condition of the theorem is false, but the first is true}, construct $\omega$ by adding an event $e$ to $\omega$ for each $\tau \in O$ (thus every obligation of $O$ is satisfied, and so naturally also every obligation in $O'$, as every obligation in $O'$ is satisfied by the same events that satisfy the obligations from $O$ that they contain), then pick one permission $\tau \in P$ such that $\{\tau\} \not\sqsubseteq P'$, and one event $e^{*}$ such that $match(\tau,e^{*})$ and not exists a $\tau'\in P'$  such that $match(\tau,e^{*})$. Such element must exist if $P \not\sqsubseteq P'$. After adding $e^{*}$ to $\omega$, $\omega$ still satisfies $p$ but not longer $p'$.

\end{proof}

\section{Related Work} \label{sec:relateds-work}

We will now present a brief overview of the evolution of ODRL and its formalisations, and how they compare with the work in this paper.  
Iannella's work in \cite{iannella2004open} 
and \cite{iannella2007open} 
presented ODRL’s early development and use cases, as well as introduced its XML structure and application to Creative Commons licensing. These studies also discussed the challenges of semantic mapping and policy interpretation. The Open Mobile Alliance and multimedia scenarios provided additional grounding in practical implementation contexts. Serr\~{a}o et al. \cite{serrao2005using} demonstrated the role of ODRL in the OpenSDRM architecture for managing streaming media and sensor data, \cite{zhang2008research} offered a broader review of ODRL’s evolution and its comparative position among rights expression languages (RELs), while \cite{kebede2018critical} highlighted the challenges in using ODRL for real-world data-sharing infrastructures.

The problem of providing an ODRL formal semantics has been tackled since the inception of the language. Pucella and Weissman defined a formal semantics of ODRL 1.1 for the problem of defining when a permission (or prohibition) follows from a set of ODRL statements \cite{pucella2006formal}. At the time, the link with RDF had not been established; ODRL was XML based and lacked a proper way to specify instances of Parties and Assets. To the best of our knowledge, the refinement of XML Schemas into RDF and ontologies was first proposed by \cite{garcia2005formalising}, while \cite{kasten2010making} proposed the first OWL modelling of ODRL, before the release of ODRL 2.0.

Steyskal and Polleres \cite{steyskal2015towards} proposed semantics for ODRL 2.1 supporting reasoning encoded in the functions that matched requests against rules in a policy. Requests contain actions, 
parties, assets and constraints. The requested actions are matched against rule actions, and return one of six values depending on the relationship between actions (\textit{match}, \textit{no match}, \textit{broader}, \textit{narrower}, \textit{part of}, \textit{required}). 
Constraints are evaluated as boolean functions, while parties and assets are matched exactly. 
Based on these matches a policy is found to be \textit{permitted}, \textit{conditionally permitted}, \textit{conditionally prohibited}, \textit{not active} or \textit{not applicable}.

De Vos et al. \cite{de2019odrl} defined the ODRL Regulatory Compliance Profile (ORCP), a custom ODRL profile, and semantics based on answer set programming where processes (series of events, in our case) are modelled as model states, and ODRL policies, which are in RDF format, are translated into \textit{fluents} --facts that are true if present and false otherwise -- following the closed-world assumption. The model and fluents would be compiled into an answer set program and solved to check for compliance between policies and regulatory frameworks. All mentions of reasoning under dependencies are left as future work, but still discussed briefly. Our approach instead saturates policies by materialising new event rules according to dependencies specified in a given ontology. 

The work in \cite{cimmino2025open} proposes a novel approach to express ODRL policies that integrates the descriptive ontology terms of ODRL with other languages that allow dynamic data handling or function evaluation. Our evaluation problem can then be cast as the execution of the program in the target language. The framework proposes abstract functions \emph{transformConstraints} and \emph{transformActions} to translate ODRL to a target programming language and provide implementations for Python and Java. They do not propose a formal semantics of the transformation functions, highlighting instead flexibility and practicality.

Recently \cite{slabbinck2025interoperable} introduced a Compliance Report Model and test suite to support systematic analysis of policy engines. They also implement an ODRL Evaluator based on the current ODRL formal semantics draft with reasoning capabilities implemented by using the EYE reasoner. However, their evaluator is incomplete due to the incompleteness of the W3C draft.

We now provide a more detailed comparison with the formalisation of ODRL semantics by Bonatti et al. \cite{bonatti2025towards}. While their work is an early stage of development, with certain components being left for future work, 
it is the most comprehensive formalisation of ODRL 2.2 to date, which builds and expands on previous work. 
They provide a declarative semantics for ODRL which focuses on a variant of the evaluation problem, where policies are used to determine compliance over traces (i.e. sequence of states). 
A core difference between our semantics and the one presented by Bonatti et al. is that ours is based on query answering and it is (1) our semantics uses a single domain of constants by abstracting away from identifiers (e.g. when two or more syntactical expression refer to the same entity); (2) it deals with inferences (such as dependencies between actions), as a separate materialisation step; (3) it does not require a distinction between ``nontemporal'' and ``temporal'' left operands, and between ``active'' or ``inactive'' rules.

At a semantic level, our formalisation differs from the one in Bonatti et al. in the following ways. In \cite{bonatti2025towards}, obligations need to be fulfilled by all possible assignees, while in ours only by one. Our choice simplifies computing compliance with obligations, as determining a complete list of all possible assignees might not always be trivial or even possible. Our semantics also provides an interpretation of obligation with consequences, while in \cite{bonatti2025towards} it is not specified under which conditions an obligation can be considered fulfilled in time (thus not requiring consequences), and when instead consequences are required. Moreover, in \cite{bonatti2025towards}, events that fulfil obligations are automatically permitted. However this is not formalised, and if any one of multiple events fulfils a single obligation, it is unclear if all of them are automatically permitted. Our semantics is more expressive, as it can also capture the case in which some events that could fulfil an obligation are not directly permitted, but might require duties or incur consequences. We also provide an interpretation of constraint operators which avoids redundancy and better aligns with the specification (e.g. differentiating between \texttt{isA}, \texttt{isPartOf} and \texttt{isAnyOf}). Lastly, we extend the study of policy comparison with an asymmetric case and motivate it in a practical scenario.

\section{Conclusion}

In this paper, we proposed a comprehensive formal semantics for the latest ODRL specification (2.2). 
Our semantics directly addresses both the problem of policy evaluation (i.e.\ whether a policy is satisfied in a state of the world) and the problem of policy comparison. For the latter problem, we define both a \emph{strict} symmetric comparison and a \emph{softer} asymmetric one. We discussed the need for both comparison approaches  in the context of  data usage policies in data marketplaces and Data Spaces, providing an answer to the increasing need for data usage regulation in the age of AI.
Our policy comparison semantics is intuitive and easy to interpret, and, due to its grounding in query containment, is easy to implement in any standard query language such as SQL or SPARQL. In particular, we provided a streamlined approach to detect both types of policy comparisons for a core subset of ODRL policies, proved its correctness and provided a reference implementation. Future work will focus on a theoretical analysis of our semantics, and extending its implementation with the aim of making it more usable for the community.

\bibliographystyle{splncs04}
\bibliography{bibliography}

\end{document}